%% file: main.tex
\documentclass[twoside]{article}

\usepackage[accepted]{aistats2019}
\usepackage[round]{natbib}

\usepackage[utf8]{inputenc} 
\usepackage[T1]{fontenc}    
\usepackage{url}            
\usepackage{booktabs}       
\usepackage{amsfonts}       
\usepackage{nicefrac}       
\usepackage{microtype}      

\usepackage{algorithm}
\usepackage{algorithmic}
\usepackage{amsmath,amssymb,amsthm}
\usepackage{graphicx}
\usepackage{wrapfig}

\newcommand{\E}{\mathbb{E}}
\newcommand{\V}{\mathbb{V}}
\newcommand{\R}{\mathbb{R}}
\newcommand{\T}{\top}
\newcommand{\given}{\mid}
\newcommand{\grad}{\nabla}
\newcommand{\evalat}{\bigg\rvert}

\newtheorem{claim}{Claim}

\begin{document}

%
\runningtitle{Can You Trust This Prediction?}

%

\twocolumn[

\aistatstitle{Can You Trust This Prediction?\\Auditing Pointwise Reliability After Learning}

\aistatsauthor{ Peter Schulam \And Suchi Saria }

\aistatsaddress{
  Department of Computer Science\\
  Johns Hopkins University\\
  \texttt{pschulam@cs.jhu.edu}
  \And
  Department of Computer Science\\
  Johns Hopkins University\\
  \texttt{ssaria@cs.jhu.edu}
} ]

\input{tex/abstract}

\input{tex/introduction}

\input{tex/related_work}

\input{tex/methods}

\input{tex/experiments}

\input{tex/conclusion}

\bibliographystyle{plainnat}
\bibliography{main}

\end{document}

%% file: tex/abstract.tex
\begin{abstract}
  To use machine learning in high stakes applications (e.g. medicine), we need tools for building confidence in the system and evaluating whether it is reliable. Methods to improve model reliability often require new learning algorithms (e.g. using Bayesian inference to obtain uncertainty estimates). An alternative is to \emph{audit} a model after it is trained. In this paper, we describe resampling uncertainty estimation (RUE), an algorithm to audit the pointwise reliability of predictions. Intuitively, RUE estimates the amount that a prediction would change if the model had been fit on different training data. The algorithm uses the gradient and Hessian of the model's loss function to create an ensemble of predictions. Experimentally, we show that RUE more effectively detects inaccurate predictions than existing tools for auditing reliability subsequent to training. We also show that RUE can create predictive distributions that are competitive with state-of-the-art methods like Monte Carlo dropout, probabilistic backpropagation, and deep ensembles, but does not depend on specific algorithms at train-time like these methods do.
\end{abstract}

%% file: tex/introduction.tex
\vspace{-8pt}
\section{Introduction}
\vspace{-8pt}

Machine learning is quickly becoming an important component of everyday software systems. This is due, in large part, to advances in tools that make it easy to express and train models. Using automatic differentiation software, a programmer can describe the ``forward pass'' of a predictive model and then train the model using gradient-based methods without much additional effort (e.g. \citealt{maclaurin2015autograd,abadi2016tensorflow}). As the barriers to building machine learning systems become lower, there is rising excitement around the idea of applying the technology in high-impact domains (e.g. \citealt{soleimani2018scalable,lipton2018does}).

Tools for quickly building machine learning systems, however, have generally outpaced the growth and adoption of tools to understand whether they are reliable. The average error on heldout training data is a common measure of model performance, but is inadequate for ensuring model reliability. For instance, low average error on heldout training data does not give assurances under \emph{data set shift} (see e.g. \citealt{candela2009dataset}). A model with low average error can also still make large \emph{pointwise errors} on individual predictions for some types of inputs (see e.g. \citealt{nushi2018towards}).

Reliability is commonly addressed during learning. In this paper, we focus on methods that improve reliability by preventing errors at test time (there are other ways to make a system reliable; \citealt{saria2019tutorial}). For instance, we can proactively account for data set shift by using learning algorithms that adjust for possible causes of the difference between train and test distributions (e.g. \citealt{schulam2017reliable,subbaswamy2018counterfactual,subbaswamy2019learning}). To detect pointwise errors, we can use specialized learning algorithms that produce models equipped with uncertainty estimates of predictions (e.g. \citealt{gal2016dropout,lakshminarayanan2017simple}). More generally, we can build confidence in a model by criticizing and refining it (e.g. \citealt{blei2014build,kim2016examples}).

In this paper, we consider the problem of improving model reliability \emph{after learning} with tools that \emph{audit} the predictions of a fixed model. To perform the audit, we introduce the Resampling Uncertainty Estimation (RUE) algorithm to calculate an uncertainty score for each prediction at test time. RUE makes minimal assumptions about the structure of the model and how it was learned. This makes RUE easy to use alongside the rich ecosystem of tools for training models (like automatic differentiation).

The most common approach in machine learning to assessing uncertainty is using Bayesian methods (e.g. \citealt{graves2011practical,blundell2015weight,gal2016dropout}). This approach is tightly coupled to the model structure and the learning algorithm. As a result, Bayesian methods cannot be applied after learning as an auditing tool. Many non-Bayesian methods like the bootstrap \citep{efron1986bootstrap} or deep ensembles \citep{lakshminarayanan2017simple} also depend on specific train-time procedures, and so cannot be applied after learning either.

There are several existing tools for improving model reliability after learning. For instance, \citet{courrieu1994three} computes the convex hull of inputs in the training data and uses it to define a region wherein outcomes can be safely interpolated by a neural network. \citet{bishop1994novelty} generalizes this idea by using density estimation to characterize regions where there are sufficient training samples to make reliable predictions. This approach can be further generalized using ideas from novelty detection where the aim is to characterize the support of a distribution without an explicit generative model (see e.g. \citealt{scholkopf2000support}). Another approach is to fit a second model to errors made by the first, which is related to stacked generalization \citep{wolpert1992stacked}.

One limitation of many existing tools for improving model reliability after learning is that they rely on Euclidean distance, which often ignores model-specific structure. For example, the predictive model may use an internal representation wherein similar points (according to the internal representation) are in fact far away in Euclidean space. This is an issue especially for deep learning, where the empirical successes are thought to be due to rich internal representations of the data that are learned to reflect task-specific invariances \citep{bengio2013representation,farabet2013learning}. Characterizing similarity among inputs for these models is difficult because it is not clear how internal representations can be extracted (see e.g. \citealt{hinton2015distilling,mahendran2015understanding}).

In Section \ref{sec:density-and-local-fit-criteria}, we show that RUE's uncertainty score reflects two important criteria to judge whether a prediction is reliable: the \emph{density criterion} and \emph{local fit criterion} \citep{leonard1992neural}. The density criterion states that a prediction at input $x$ is reliable if there are samples in the training data that are similar to $x$. The local fit criterion states that a prediction at $x$ is reliable if the model has small error on samples in the training data similar to $x$. These criteria hinge upon a definition of similarity, and we show that RUE implements these criteria using a model-dependent inner product to measure similarity.


\textbf{Contributions.} We describe Resampling Uncertainty Estimation (RUE), an algorithm used to audit the pointwise reliability of predictions by computing an uncertainty score. To quantify uncertainty, the RUE algorithm uses the gradients and Hessian of the model's loss on training data (Line 2 in Algorithm 1) and bootstrap samples (Line 7) to produce an ensemble of predictions for a test input (Lines 9-11). Intuitively, this estimates the amount that the prediction would change had the model been fit on different data sets drawn from the same distribution as the training data. This intuition has connections to the classical bootstrap \citep{efron1986bootstrap}, and in Section 3.1 we show that RUE can be derived as an approximation to this procedure that does not require fitting multiple models. We also show that RUE implements the density and local fit criteria using a model-dependent inner product as a measure of similarity. Experimentally, we show that RUE more effectively detects inaccurate predictions than existing methods for improving reliability after learning. Moreover, we show that RUE's uncertainty score can create predictive distributions that are well-calibrated and competitive with state-of-the-art \emph{integrated} methods (i.e. methods like \citealt{hernandez2015probabilistic,gal2016dropout}; and \citealt{lakshminarayanan2017simple}) that estimate uncertainty during learning.

%% file: tex/related_work.tex
\vspace{-8pt}
\section{Related Work}
\vspace{-8pt}


The dominant approach to modeling uncertainty in machine learning is to use Bayesian inference. Gaussian processes are a flexible class of Bayesian nonparametric priors over functions and provide a natural framework for reasoning about posterior predictive distributions at test time given training data \citep{rasmussen2006gaussian}. Bayesian neural networks are another important class of regression models, but have historically been much more challenging to work with given the nonlinearities in the likelihood function. Early important work on this subject was done by \citet{buntine1991bayesian} and by \citet{mackay1992practical}, who proposed using a Laplace approximation to the posterior distribution in order to obtain posterior predictive distributions and estimates of the evidence (marginal probability of the data) for model selection. A key challenge with this technique is that it requires computing and inverting the Hessian with respect to the network parameters, which can be prohibitively expensive for larger neural networks. There has, however, been recent work in scalable approximations to the Hessian for second-order optimization (e.g. \citealt{martens2015optimizing,botev2017practical}), which could make the Laplace approximation to the posterior feasible \citep{ritter2018scalable}.

Efforts to scale Bayesian inference to modern neural networks have revolved around mean field variational approximations to the posterior distribution (e.g. \citealt{graves2011practical,blundell2015weight}). There has also been work to apply expectation propagation to Bayesian neural networks \citep{hernandez2015probabilistic,hernandez2016black}. Dropout, a non-Bayesian technique for regularizing neural networks \citep{srivastava2014dropout}, was recently shown to approximate a variational approximation of the posterior \citep{gal2016dropout}, and the insight led to an approach for computing posterior predictive distributions using a simple simulation algorithm that leverages existing software for fitting neural networks that implements dropout layers. There has also been recent work to scale MCMC for posterior inference by designing approximate transition operators that trade off between fidelity to the desired stationary distribution and computational cost (e.g. \citealt{welling2011bayesian,ahn2012bayesian}). Although the Bayesian approach is conceptually appealing, it is difficult to place meaningful priors on neural networks (see e.g. \citealt{buntine1991bayesian}).

Many authors have explored alternatives to Bayesian methods for modeling uncertainty in predictive models. Early work by \citet{leonard1992neural} explored heuristic scores that reflect the density and local fit principles. \citet{bishop1994novelty} proposes a technique for identifying ``novel'' inputs as a way to detect when a model is likely to be unreliable. \citet{hooker2004diagnosing} proposes a tree-based density estimator for detecting novel inputs. More recently, \citet{hendrycks2016baseline} proposed a simple heuristic using softmax outputs to detect misclassifications and out-of-distribution samples for classification problems. They also propose an ``abnormality module'' for detecting novel inputs that is similar in spirit to the validity index network proposed by \citet{leonard1992neural}. \citet{guo2017calibration} investigate whether softmax outputs are calibrated estimates of the conditional distribution $p(y \given x)$. They find that the softmax probabilities are not calibrated in large neural networks and propose an approach related to Platt scaling \citep{platt1999probabilistic} to address the issue. Ensembles can also be used to express uncertainties \citep{breiman1996bagging,lakshminarayanan2017simple}.


The bootstrap \citep{efron1986bootstrap} is a classic approach to estimating uncertainty that uses Monte Carlo simulation to estimate the sampling distribution of arbitrary functions of a data set. It cannot be used after learning because we must fit many models at train time rather than just one (this also makes it computationally expensive).
We show that RUE is an approximation of the boostrap by building upon sensitivity analysis ideas pioneered by \citet{cook1986assessment} that were recently extended to modern machine learning problems by \citet{koh2017understanding} in order to shed light on a model's provenance and improve interpretability.

There are several other threads of research related to improving model reliability. One issue that in classification problems is when the current set of labels that a model assigns to inputs is not exhaustive. If this occurs, a model that automatically identifies and learns to classify new categories can make the system more robust (e.g. \citealt{bendale2016towards,liu2018open}). It is also useful for a machine learning system to implement a policy for rejecting an input due to uncertainty. Such a policy can be easy to implement if the model outputs well-calibrated probability distributions (e.g. \citealt{chow1970optimum}), but models can also be trained to directly minimize a risk function incorporating the reject option (e.g. \citealt{herbei2006classification,bartlett2008classification,grandvalet2009support}).

%% file: tex/methods.tex
\vspace{-10pt}
\section{Resampling Uncertainty Estimation}
\vspace{-10pt}


We assume the classical supervised learning setup. Let $(x, y)$ denote an input-output pair, where $x$ lies in $\mathcal{X} \subseteq \R^p$ and $y$ lies in $\mathcal{Y} \subseteq \R$. Our learning algorithm is given $n$ training examples $D \triangleq \{ (x_i, y_i) \}_{i=1}^n$, which we assume are drawn iid from an unknown distribution $P$. Let $\mathcal{H}$ denote a hypothesis space with members $f$ indexed by a vector of parameters $\theta \in R^d$. We assume that our algorithm learns by minimizing an objective
\begin{equation}
  \label{eq:objective}
  J_D(\theta) \triangleq \sum_{i=1}^n \ell(y_i, f(x_i, \theta)) + R(\theta),
\end{equation}
where $\ell$ is a loss function and $R(\theta)$ is a regularizer. To apply RUE, we assume that we have access to (1) the matrix $L \in \R^{d \times n}$ where each column $i$ contains the gradient $\grad_\theta \ell(y_i, f(x_i, \hat{\theta})$ and (2) the Hessian $H = \grad^2_\theta J_D(\hat{\theta})$, where $\hat{\theta}$ is the parameter vector returned by the learning algorithm. RUE does not depend on assumptions about how the solution $\hat{\theta}$ is obtained. Moreover, we do not assume that $\hat{\theta}$ is a local or global minimum.

Algorithm \ref{alg:uncertainty-scores} shows the full steps used to compute uncertainty estimates using RUE for a set of $m$ test points $\{x_j\}_{j=1}^m$. Note that RUE must invert the Hessian $H$, but because we do not assume anything about the convergence of the learning algorithm to a local minimum of Equation \ref{eq:objective}, there is no guarantee that $H$ is positive definite. As in \citet{koh2017understanding}, we can add a dampening term $\lambda I$ to make the Hessian invertible, which is equivalent to adding $L_2$ regularization to the learning objective. We denote this adjusted Hessian using $\tilde{H} \triangleq H + \lambda I$, and discuss a strategy to select $\lambda$ in the experiments section below.

\vspace{-5pt}
\begin{algorithm}[H]
  \caption{Resampling Uncertainty Estimation}
  \begin{algorithmic}[1]
    \REQUIRE num. training samples $n$, ensemble size $b$, Hessian of loss $H \in \R^{d \times d}$, loss gradient matrix $L \in \R^{d \times n}$, test cases $\{x_j\}_{j=1}^m$, dampen term $\lambda$.
    \STATE $\tilde{H} \leftarrow H + \lambda * \textsc{Eye}(d)$
    \STATE $A \leftarrow \tilde{H}^{-1} L$
    \STATE $w_0 \leftarrow \textsc{Ones}(n)$
    \STATE $p_0 \leftarrow w_0 / n$
    \STATE $\hat{Y} \leftarrow \textsc{Zeros}(b, m)$
    \FOR{$i=1$ to $b$}
    \STATE $w_i \sim \text{Multinomial}(n, p_0)$
    \STATE $\theta^*_i \leftarrow \hat{\theta} - A (w_i - w_0)$
    \FOR{$j=1$ to $m$}
    \STATE $\hat{Y}[i, j] \leftarrow \textsc{Predict}(\theta^*_i, x_j)$
    \ENDFOR
    \ENDFOR
    \STATE \textbf{return} $\textsc{ColumnWiseVariance}(\hat{Y})$
  \end{algorithmic}
  \label{alg:uncertainty-scores}
\end{algorithm}
\vspace{-5pt}

\subsection{RUE as Approximation to the Bootstrap}
\label{sec:derivation-from-bootstrap}

In this section, we connect RUE to Efron's bootstrap procedure. The bootstrap approximates the sampling distribution of an estimate by repeatedly simulating \emph{bootstrap samples}, which are new data sets of size $n$ created by sampling with replacement from the uniform distribution over the original data set. To bootstrap a supervised learning algorithm, one would need to sample $b$ bootstrap datasets and run the learning procedure from scratch each time. When learning on the original sample, let $\hat{\theta}$ denote the learned parameters that satisfy
\begin{equation}
  \grad J_D(\hat{\theta})
  = \sum_{i=1}^n \grad_\theta \ell_i(\hat{\theta}) +
  \grad_\theta R(\hat{\theta}) = C
\end{equation}
for some constant $C$, where we have used $\ell_i(\hat{\theta})$ as shorthand for $\ell(y_i, f(x_i, \hat{\theta}))$.  Let $w_i$ denote a multiplicative weight for the $i^{\text{th}}$ training sample in the objective function. When $w_i = 1$ for all $i$, then the objective can be rewritten to depend on the weights
\begin{equation}
  J_D(\theta ; w) = \sum_{i=1}^n w_i \ell_i(\theta) + R(\theta).
\end{equation}
The bootstrap is equivalent to choosing new weights $w_i$ that count the number of times a particular example is included in the bootstrap sample and then refitting the objective. Considering our objective function as a function of both the model parameters $\theta$ and sample weights $w_0 \triangleq \{w_i\}$, the following equality holds at the solution $\hat{\theta}$:
\begin{equation}
  \grad_\theta J_D(\hat{\theta} ; w_0) \triangleq K_D(\hat{\theta} ; w_0) = L
  w_0 + \grad_\theta R(\hat{\theta}) = C.
\end{equation}
As a reminder, $L \in \R^{d \times n}$ is a matrix where the $i^{\text{th}}$ column is the gradient of the loss of the $i^{\text{th}}$ sample $\grad_\theta \ell_i(\hat{\theta})$, and $w_0 \in \R^n$ is a vector of ones.

If we assume that $K_D(\hat{\theta} ; w_0)$ is smooth in $\theta$ (this amounts
to assuming that the loss and regularizer are twice continuously
differentiable), we can use the implicit function theorem to conclude that there
exists a local function such that $\phi(w_0) = \hat{\theta}$. In other words,
there is a map from the sample weights to learned model parameters. If we knew
this map, then the bootstrap procedure would be straightforward: sample new
weights $w^*$ and map them to the parameters $\phi(w^*) = \theta^*$. Our
approach approximates this function using a first-order Taylor expansion around
$\hat{\theta}$, which means we can estimate the parameters we would obtain with
a bootstrap sample $w^*$ as
\begin{equation}
  \theta^* \approx \hat{\theta} +
  \frac{\partial \phi}{\partial w} \evalat_{w_0} (w^* - w_0)
  = \hat{\theta} - H^{-1} L (w^* - w_0),
\end{equation}
where the equality above also follows from the implicit function theorem (recall
that $H$ is the Hessian of the objective $J_D$ evaluated at $\hat{\theta}$). The
Jacobian of $\phi$ can be computed once (this is the matrix $A$ in Algorithm
\ref{alg:uncertainty-scores}), then sampling $\theta^*$ only requires drawing
new weights $w^*$ and a matrix multiplication. To approximate the sampling
distribution of the model's prediction at a new point $x^*$, we compute the
predictions made by a collection of models with parameters
$\{\theta^*_i\}_{i=1}^b$.

\subsection{Density and Local Fit Criteria}
\label{sec:density-and-local-fit-criteria}

\newcommand{\rue}[1]{\hat{\sigma}_{\text{RUE}}(#1)}
\newcommand{\ruesq}[1]{\hat{\sigma}^2_{\text{RUE}}(#1)}

\begin{figure*}[t]
  \vskip 0.2in
  \begin{center}
    \includegraphics[width=0.8\linewidth]{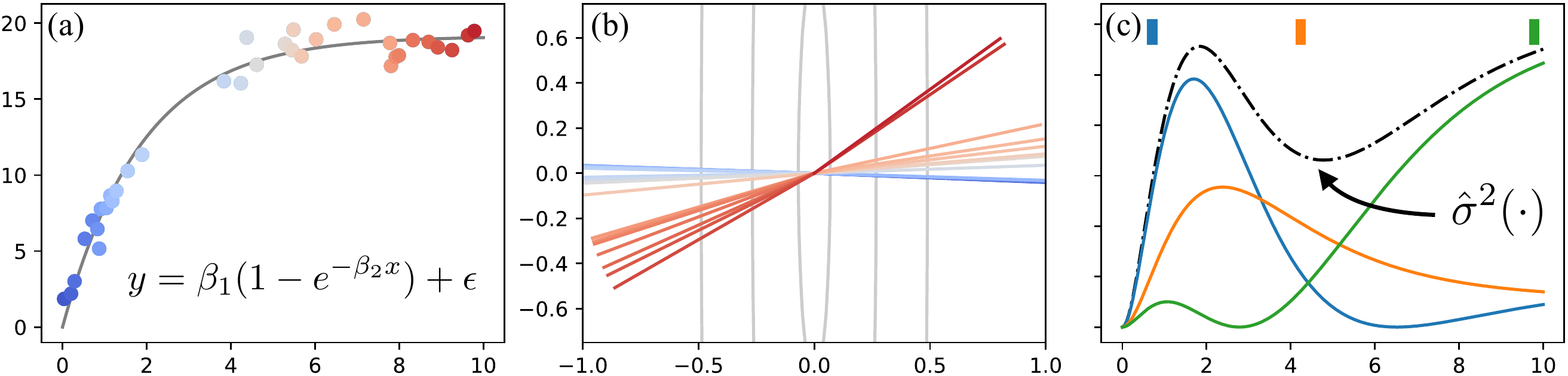}
    \vskip -0.05in
    \caption{Simulated data demonstrating the model-dependent density and local
      fit principles for Resampling Uncertainty Estimation (RUE).}
    \label{fig:density-illustration}
  \end{center}
  \vskip -0.2in
\end{figure*}

RUE outputs uncertainty scores $\rue{x}$ by computing the standard deviation of
predictions sampled within the loop on lines 6-12 in Algorithm
\ref{alg:uncertainty-scores}. In this section, we argue that these uncertainty
scores check both the density and local fit criteria.  Moreover, we show
that RUE uses a model-dependent similarity measure. For this analysis, we will
assume that the loss $\ell$ is the negative log likelihood of an exponential
family:
\begin{equation}
  \ell(y_i, f)= -\log h(y_i) - t(y_i) f + a(f),
\end{equation}
where $h(y_i)$ is the base measure of the exponential family, $t(y_i)$ is the
natural sufficient statistic, the predicted value $f$ is the natural parameter,
and $a(f)$ is the log normalizing constant. Our argument rests upon the
following claim.
\begin{claim}
  As the number of training samples $n$ grows, we can approximate $\ruesq{x}$ as
  \begin{equation}
    \label{eq:rue-approximation}
    \ruesq{x} \approx
    \sum_{i=1}^n r_i^2 k_{\text{RUE}}(x, x_i),
  \end{equation}
  where $r_i = \E[t(y_i) \given f] - t(y_i)$ and $k_{\text{RUE}}$ is a positive
  definite kernel function with definition
  \begin{equation}
    \label{eq:rue-kernel}
    k_{\text{RUE}}(x_1, x_2) =
    \left[
      \grad^\T_\theta f(x_1, \hat{\theta})
      \tilde{H}^{-1}
      \grad_\theta f(x_2, \hat{\theta})
    \right]^2.
  \end{equation}
\end{claim}
\begin{proof}
  Let $\theta^*$ denote a parameter vector resampled according to lines 7-8 in
  Algorithm \ref{alg:uncertainty-scores}, then
  \begin{align}
    \ruesq{x}
    &= \V[ f(x, \theta^*) ] \\
    &\approx \V[
      f(x, \hat{\theta}) +
      \grad^\T_\theta f(x, \hat{\theta}) (\theta^* - \hat{\theta})
    ] \\
    &= \grad^\T_\theta f(x, \hat{\theta}) \V[ \theta^* ]
    \grad^\T_\theta f(x, \hat{\theta})
  \end{align}
  To compute the variance of $\theta^*$, recall that $\theta^* = \hat{\theta} -
  \tilde{H}^{-1} L (w - w_0)$, where $w \sim \text{Multinomial}(n, p_0)$, $p_0
  \in \R^n$ with all entries $n^{-1}$, and $w_0 \in R^n$ is a vector of
  ones. The covariance matrix is therefore $\V[\theta^*] = \tilde{H}^{-1} L
  \V[w] L^\T \tilde{H}^{-1}$. Because $w$ is a multinomial random variable, its
  covariance matrix is $\nicefrac{n}{n-1}$ along the diagonal, and
  $-\nicefrac{1}{n}$ on the off diagonal. Therefore, as the number of samples $n
  \to \infty$ we have $\V[w] \to I$, where $I$ is the identity matrix.  Now,
  recall that $L \in \R^{d \times n}$ is a matrix where column $i$ contains
  $\grad_\theta \ell(y_i, f(x_i, \hat{\theta})$. Using the chain rule, we can
  write $L = G R$, where $G \in \R^{d \times n}$ is a matrix where column $i$
  contains $\grad_\theta f(x_i, \hat{\theta})$ and $R \in \R^{n \times n}$ is a
  diagonal matrix where the $i^{\text{th}}$ diagonal is $r_i \triangleq \grad_f
  \ell(y_i, f(x_i, \hat{\theta}) = \E[t(y_i) \given f] - t(y_i)$.\footnote{The
    last equality follows from $\grad_f a(f) = \E[t(y) \given f]$}
  As $n$ gets large, we have
  \begin{align*}
    \ruesq{x}
    &\approx \grad^\T_\theta f(x, \hat{\theta}) \tilde{H}^{-1} L
    L^\T \tilde{H}^{-1} \grad_\theta f(x, \hat{\theta}) \\
    &= \grad^\T_\theta f(x, \hat{\theta}) \tilde{H}^{-1} G R^2
    G^\T \tilde{H}^{-1} \grad_\theta f(x, \hat{\theta}).
  \end{align*}
  Definine the kernel as in \ref{eq:rue-kernel} and write the matrix
  expression as a sum of scalars to complete the argument.
\end{proof}

Note that the approximation of $\ruesq{x}$ resembles a kernel smoothing estimate of the exponential family residual at a test input $x$. Interpeting the kernel function $k_{\text{RUE}}(x, x_i)$ as a measure of similarity between the test point $x$ and training point $x_i$, we can see that $\rue{x}$ reflects the local fit criterion. In other words, the uncertainty score will be smaller (larger) if the model's predictions have small (large) errors on the training samples that are similar to $x$.

We can also use Equations \ref{eq:rue-approximation} and \ref{eq:rue-kernel} to show that RUE implements the density criterion. Let $\mathcal{N}(x)$ denote the set of influential neighbors of $x$; i.e. those with the highest similarity that contribute most to the uncertainty score $\rue{x}$. In Equation \ref{eq:rue-kernel}, we see that the similarity $k_{\text{RUE}}(x, x_i)$ depends on the gradients $\grad_\theta f(x, \hat{\theta})$ and $\grad_\theta f(x_i, \hat{\theta})$, but also on the dampened Hessian $\tilde{H} \triangleq H + \lambda I$. In particular, we see that the gradients will be projected into the space spanned by the eigenvectors of $\tilde{H}$, and then scaled by the corresponding inverse eigenvalues. The eigenvectors with largest eigenvalues will point in directions with high curvature, implying that the training data strongly determine the model parameters along that direction. We see that $\rue{x}$ will be smaller if $x$'s gradient and those of its neighbors $\mathcal{N}(x)$ lie along a direction with high curvature. If we consider curvature to be a proxy for density, then we see that RUE also implements a form of the density criterion.

Finally, we note that the similarity metric $k_{\text{RUE}}(x, x_i)$ depends on
both the structure of the model and the learned parameter values $\hat{\theta}$
through the gradient and Hessian. We therefore see that RUE uses a measure of similarity that is model-dependent for both the local fit and density criteria, which is a distinguishing feature among tools for improving reliability after learning.

\subsubsection{Simulated Data Illustration}

To make the connection between RUE and the density and local fit criteria
concrete, we use a small simulated data set shown in Figure
\ref{fig:density-illustration}. We consider a 1-dimensional nonlinear regression
problem, where the mean function is shown in Figure
\ref{fig:density-illustration}a (grey line) and is defined as $\E[y \given x] =
\beta_1 (1 - \exp^{-\beta_2 x})$. The data are normally distributed around the
mean and are shown as points in the same panel. In Figure
\ref{fig:density-illustration}b, we plot the directions of the gradient of the
likelihood $\grad_\theta \ell(y_i, f(x_i, \hat{\theta}))$ evaluated at the MLE
for each of the training samples projected on to the eigenvectors of the
Hessian. The grey contours show the shape of the quadratic defined by the
Hessian in this rotated space. The colors of the points in Figure
\ref{fig:density-illustration}a and of the directions in Figure
\ref{fig:density-illustration}b show the correspondence between training samples
and their gradients. We highlight two observations about panel (b). First, the
points whose gradient vectors lie along similar directions are not necessarily
close on the x-axis in panel (a). The blue and tan points, for example, have
similar gradient vectors but the tan points are closer to the red than to the
blue points in panel (a). Second, the gradient vectors lie along directions with
high curvature (i.e. moving left to right in panel (b)), suggesting that
curvature does indeed reflect the density principle in this space. Finally, in
Figure \ref{fig:density-illustration}c we show how the model-dependent kernel
and error terms $r_i^2$ are smoothed to produce an uncertainty estimate. Each
summand in Equation \ref{eq:rue-approximation} is an error $r_i^2$ weighted by a
basis function $k_{\text{RUE}}(\cdot, x_i)$. Panel (c) plots these basis
functions for three samples from the training data (colored curves, location of
training sample is indicated by colored tick at the top of the plot) and also
shows the uncertainty score $\ruesq{\cdot}$ as a function of $x$ (dotted
curve). Higher values of the basis functions reflect points for which the
associated training sample is considered relevant.

\vspace{-8pt}
\subsubsection{Related Uncertainty Scores}
\vspace{-4pt}

\paragraph{Laplace approximation.}
As a point of comparison, the variance of $\theta$ under the Laplace approximation to the posterior is $\V[\theta] = [\grad_\theta^2   J_D(\hat{\theta})]^{-1} = H^{-1}$. If we linearize the model at $\hat{\theta}$ we can approximate the posterior predictive variance at a test point $x$ as $\hat{\sigma}^2_{\text{Laplace}}(x) \approx \grad^\T_\theta f(x, \hat{\theta}) \tilde{H}^{-1} \grad_\theta f(x, \hat{\theta})$. We see that the Laplace approximation measures uncertainty using the norm of the test point according to a model-specific inner product. The score is small when the gradient of the input points in a direction of high curvature, and so reflects the density criteria, but not local fit.

\vspace{-8pt}
\paragraph{Sandwich covariance estimator.}

RUE has connections to robust statistics. When the learning objective is the log likelihood of independent samples, the negative Hessian of the objective is known as the Fisher information matrix, and it is often used to approximate the covariance matrix under the sampling distribution of a maximum likelihood estimator \citep{tibshirani1996comparison}. In likelihood theory, an alternative estimator of the covariance matrix of the parameters that is robust to model misspecification is the sandwich covariance estimator \citep{kent1982robust}, which is $g^\T(x) H^{-1} L L^\T H^{-1} g(x)$.
Recall that we arrived at a similar expression in Equation \ref{eq:rue-approximation} by letting $n \to \infty$. This draws a connection between RUE and robust statistics \citep{huber2009robust}. In our experiments, we find that RUE outperforms the Laplace approximation and these connections to robust statistics may help to develop theory that explains this observation.


\subsubsection{Scalability}
\label{sec:scalability}

One concern with RUE is that it may be too computationally demanding to use in practice. In particular, the algorithm requires us to store the Hessian (requiring $O(d^2)$ space), invert the Hessian (requiring $O(d^3)$ operations), and store all gradients of the training data (requiring $O(nd)$ space). We address the issue related to storing all gradients of the training data by noting that it may be possible to retain only a small fraction of the original data set and still obtain high quality reliability scores. Note that in step 8 of Algorithm \ref{alg:uncertainty-scores}, we sample a parameter using the update $\theta^*_i \leftarrow \hat{\theta} - A(w_i - w_0)$,
where $A = \tilde{H}^{-1} L$. This sampling step is a Newton-Raphson update using a resampled data set. Just as there are diminishing returns in terms of accuracy when computing stochastic gradient updates with larger minibatches \citep{bousquet2008tradeoffs}, there may be a similar relationship between the amount of training data retained and the quality of the reliability scores computed using RUE. In our experiments, we consider smaller data sets and so leave this extension for future work.

To address storage and computation issues related to the Hessian, we note that we do not need to explicitly store the Hessian so long as we have access to the model code and can run backwards-mode automatic differentiation, which allows us to efficiently compute Hessian-vector products \citep{pearlmutter1994fast}. Moreover, we can perform the required inversion approximately using conjugate gradients (which only requires Hessian-vector products), or can use more sophisticated techniques discussed by \citet{koh2017understanding}. Alternatively, we could use recently developed approximations to the Hessian such as K-FAC \citep{martens2015optimizing}. We leave exploration and evaluation of these approximations for future work.

%% file: tex/experiments.tex
\begin{figure*}[ht]
  \vskip 0.2in
  \begin{center}
    \includegraphics[width=0.85\linewidth]{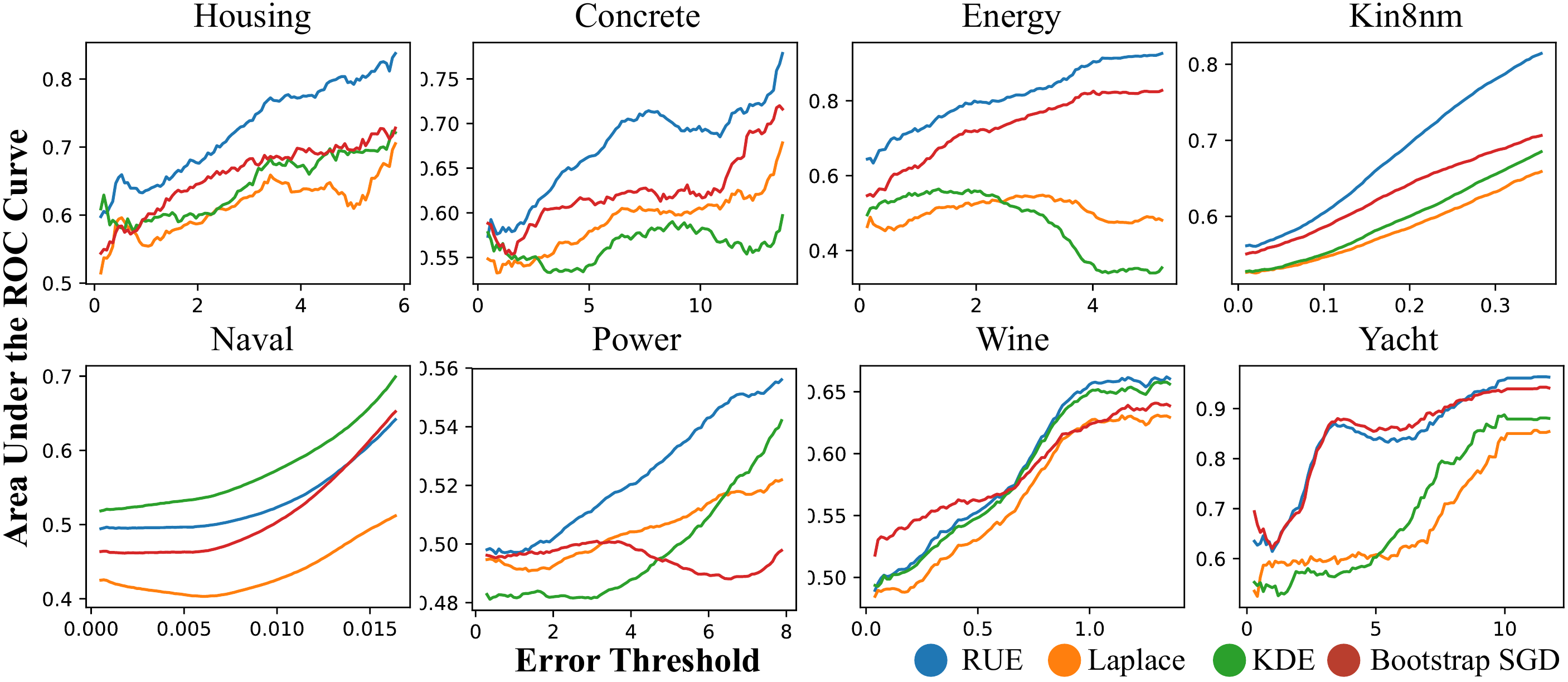}
    \vskip -0.1in
    \caption{Comparison of uncertainty scores across eight benchmark regression
      data sets. Each figure plots area under the ROC curve as a function of the
      error threshold used to define correct and incorrect predictions (curves
      that grow higher and faster from left to right are better).}
    \label{fig:results}
  \end{center}
  \vskip -0.2in
\end{figure*}

\vspace{-5pt}
\section{Experiments}
\vspace{-5pt}

\textbf{Overview.} We first evaluate RUE on an auditing task where the goal is to detect when model predictions will be wrong. We compare to existing methods for improving reliability after learning, and show that RUE generally outperforms them. Next, we show that RUE can create predictive distributions over outputs that are competitive with state-of-the-art \emph{integrated} methods that rely on particular algorithms at train time (e.g. Bayesian inference).

\vspace{-6pt}
\subsection{Error Detection}
\vspace{-6pt}

In our first experiment, we compare how well RUE and three baseline methods for improving reliability after learning are able to detect when a model's prediction will be far from the true value. The first alternative is the Laplace approximation \citep{mackay1992practical}, which uses the curvature of the learning objective as the precision matrix of a multivariate normal over the model parameters. We draw samples from this multivariate normal distribution and compute an ensemble of predictions to obtain a predictive distribution.
The second baseline uses a kernel density estimator (KDE) to estimate the marginal distribution over inputs in the training sample. The uncertainty score of a test prediction is the negative density of the input; test cases that have higher density lie closer to the training data in Euclidean space and so should be more certain according to the density principle. We use an isotropic Gaussian kernel for the KDE, which depends on a single bandwidth parameter that we estimate using 5-fold cross validation on the training data. Finally, we compare to a method similar to RUE but that uses only first-order rather than second-order information. This baseline computes an ensemble as in Algorithm \ref{alg:uncertainty-scores} but on Line 8 the new parameter $\theta^*_i$ is set to $\theta - \eta L w_i$, which is the parameter vector obtained by taking a single gradient step using a bootstrapped sample with step size $\eta$. We use the step size $\eta = 0.001$, which is the base learning rate during training (more information is in paragraph \textbf{Model} below). We refer to this baseline as Bootstrap SGD.
Both RUE and the Laplace approximation must invert the Hessian. To choose a dampening factor $\lambda$, we compute the eigenvalues of $H$ using \textsc{LAPACK}'s \texttt{syevd} routine and choose $\lambda$ so that the smallest eigenvalue of $\tilde{H}$ is at least one.

\input{tex/nll_table.txt}

\vspace{-0.1in}
\paragraph{Data.}
We run our experiments on eight common benchmark regression data sets
(e.g. \citealt{hernandez2015probabilistic,gal2016dropout,lakshminarayanan2017simple}). The data sets we use are: Boston housing, concrete compression strength, energy efficiency, robot arm kinematics, naval propulsion, combined cycle power plant, red wine quality, and yacht hydrodynamics.\footnote{These data sets are all   available from \texttt{https://archive.ics.uci.edu/ml/index.php}} For each data set, we sample twenty train-test splits. For the smaller data sets (housing, concrete, energy, and yacht), we sample $90\%$ of the points for training. The remaining data sets are larger and so we use smaller percentages of the data set for each training split in order to simulate a small data regime, where uncertainty scores are most useful. For the larger data sets we sample approximately $600$ training points per split and test on the remaining points.

\vspace{-0.1in}
\paragraph{Model.}
To control for model structure, learning objective, and optimization algorithm across all data sets and uncertainty estimation algorithms, we fix a single feedforward neural network architecture with a single hidden layer comprised of 50 hidden units, softplus nonlinearities, and $L_2$ loss. We regularize with weight decay and set the penalty $\alpha = 1$. Moreover, we use a single set of learning hyperparameters: we learn for a total of $500$ epochs with $128$ samples per minibatch, and use the default settings ($b_1 = 0.9, b_2 = 0.999$) of \textsc{Adam} \citep{kingma2014adam} with a fixed learning rate of $0.001$. We fit a single model for each split of each data set, and compute uncertainty scores for the predictions made by that model on the test set using the four alternative uncertainty score methods (RUE, Laplace, KDE, and Bootstrap SGD).

\vspace{-0.1in}
\paragraph{Evaluation.}
Let $y \in \R$ denote the true output for input $x$ and let $\hat{y} \in \R$ be the model's prediction. For a given error tolerance $\tau$, our goal is to use uncertainty scores to detect whether the prediction will be correct ($|y - \hat{y}| \leq \tau$) or incorrect ($|y - \hat{y}| > \tau$). We use the area under the ROC curve (AUC) to evaluate each uncertainty score's ability to detect incorrect predictions across all test splits for a given data set. Figure \ref{fig:results} shows AUC as a function of error tolerance $\tau$ for each data set and for each uncertainty score. We sweep across thresholds between the $5^{\text{th}}$ and $95^{\text{th}}$ percentiles of the empirical distribution over absolute errors.

\vspace{-0.1in}
\paragraph{Results.}
RUE is better at detecting errors than Laplace, KDE, and Bootstrap SGD for most data sets and most error tolerances. Note that the range of AUC scores varies widely across data sets. For instance, the RUE scores are able to discriminate well across all thresholds in the energy data set. The power data set, however, is challenging; all AUCs are $< 0.56$.

\vspace{-5pt}
\subsection{RUE Predictive Distributions}
\label{sec:calibration}
\vspace{-5pt}

Existing methods for quantifying uncertainty output probabilities. Although the methods we compare to in the first experiments above do not output probability distributions, there is a natural way to create one for the methods that use ensembles of model predictions (RUE, Laplace, and Bootstrap SGD). When the observation noise is additive---that is, $y = f(x) + \epsilon$---the posterior variance of a prediction at input $x$ for a Bayesian model is the posterior variance of $f(x)$ plus the variance of the observation noise $\epsilon$. For the ensemble methods, the uncertainty scores can be seen as approximations of the variance of $f(x)$, which we denote using $\hat{\sigma}^2(x)$. We can also estimate the variance of $\epsilon$ using the residuals on the training data, which we denote using $\hat{\nu}^2$. A natural predictive distribution is a Gaussian centered at the prediction $f(x)$ under the learned model parameters $\hat{\theta}$ with variance $\hat{\sigma}^2(x) + \hat{\nu}^2$. We use this predictive distribution to compute negative log likelihoods (NLL) and compare the methods used after learning to state-of-the-art alternatives that estimate uncertainty during training using specialized objectives (\citealt{gal2016dropout,hernandez2015probabilistic,lakshminarayanan2017simple}), which we refer to as \emph{integrated} methods.

Table \ref{tab:nll} shows the results of this experiment. The negative log likelihoods for the integrated methods are the results reported in Table 1 of \citet{lakshminarayanan2017simple}. We find that RUE has better NLL than Laplace and Bootstrap SGD on all but two data sets. The improvements over Bootstrap SGD, however, are small, but recall that RUE generally outperformed Bootstrap SGD in the detection experiments shown in Figure \ref{fig:results}. We also compare RUE to three integrated methods: Probabilistic Backpropagation \citep{hernandez2015probabilistic}, Monte Carlo Dropout \citep{gal2016dropout}, and Deep Ensembles \citep{lakshminarayanan2017simple}. We find that RUE is competitive with these methods (e.g., on Housing, Power and Wine), but does not outperform them. We are not surprised by this result because integrated methods jointly fit predictions and uncertainty using a specialized objective. Furthermore, we suspect that RUE's negative log likelihood would be further improved if the learning rates and regularization weights were tuned using cross validation for each data set (these choices were fixed upfront at a single setting for all data sets in order to control for effects of learning algorithms).

%% file: tex/nll_table.txt
\begin{table*}[t]
\fontsize{8.5}{10.2}\selectfont
\centering

\caption{Negative log likelihoods for integrated and after-learning methods. Details are in Section \ref{sec:calibration}}.
\begin{tabular}{l | rrr | rrr}
                  & \multicolumn{3}{c}{Integrated Methods}                      & \multicolumn{3}{c}{Subsequent-to-Training Methods}                              \\
\textbf{Data Set} & \textbf{PBP} & \textbf{MCDropout} & \textbf{Deep Ensembles} & \textbf{RUE}          & \textbf{Laplace} & \textbf{Bootstrap SGD} \\
		  \hline
Housing           & 2.57 (0.09)  & 2.46 (0.25)        & 2.39 (0.25)             & \textbf{2.54 (0.06)}  & 4.12 (0.01)      & 2.55 (0.04)            \\
Concrete          & 3.16 (0.02)  & 3.04 (0.09)        & 3.05 (0.15)             & \textbf{3.36 (0.02)}  & 4.09 (0.01)      & 3.40 (0.01)            \\
Energy            & 2.04 (0.02)  & 1.99 (0.09)        & 1.40 (0.17)             & \textbf{2.27 (0.02)}  & 3.66 (0.01)      & 2.32 (0.01)            \\
Kin8nm            & -0.90 (0.01) & -0.95 (0.03)       & -1.19 (0.02)            & \textbf{-0.34 (0.00)} & 0.11 (0.00)      & -0.32 (0.00)           \\
Naval             & -3.73 (0.01) & -3.80 (0.05)       & -4.17 (0.05)            & \textbf{-3.33 (0.00)} & -2.74 (0.00)     & -3.01 (0.00)           \\
Power             & 2.84 (0.01)  & 2.80 (0.05)        & 2.78 (0.04)             & \textbf{2.88 (0.00)}  & 4.51 (0.00)      & 2.90 (0.00)            \\
Wine              & 0.97 (0.01)  & 0.93 (0.06)        & 0.93 (0.12)             & \textbf{1.01 (0.00)}  & 1.47 (0.00)      & \textbf{1.01 (0.00)}   \\
Yacht             & 1.63 (0.02)  & 1.55 (0.12)        & 1.18 (0.19)             & 3.02 (0.05)           & 3.88 (0.01)      & \textbf{2.99 (0.04)}
\end{tabular}
\label{tab:nll}
\end{table*}

%% file: tex/conclusion.tex
\vspace{-12pt}
\section{Conclusion}
\vspace{-12pt}


We introduced RUE, an algorithm that audits a model by estimating the uncertainty of predictions without changing the learning algorithm. This makes it easy to use RUE alongside the rich ecosystem of existing tools for training machine learning models (e.g. automatic differentiation libraries). This is in contrast to methods like Bayesian inference that depend on specific learning algorithms to estimate uncertainty. By making it easier to estimate uncertainty, we suspect that after-learning auditing methods like RUE will help to increase the adoption of machine learning in high-stakes domains such as medicine.
We evaluated RUE on an error detection task, and showed that it outperforms existing tools for improving model reliability after learning. Moreover, we showed that RUE generates predictive distributions that are competitive with state-of-the-art integrated methods that use specialized training algorithms to quantify uncertainty (e.g. Bayesian inference). We also showed that RUE's uncertainty score reflects the density and local fit criteria using a model-dependent similarity measure. In future work we can define new audits by leveraging different principles. For instance, methods for interpreting model predictions (e.g. \citealt{ribeiro2016should}) or performing sensitivity analyses (e.g. \citealt{koh2017understanding}) after learning may inspire new criteria for reliability that could be integrated into different auditing tools.